\documentclass{article}
\usepackage[top=1in,bottom=1in,left=1in,right=1in]{geometry}  

\usepackage[utf8]{inputenc}
\usepackage{graphicx, xcolor, amsmath}
\usepackage{setspace}
\usepackage[colorlinks,allcolors=cyan!70!black]{hyperref}

\usepackage{enumitem}

\usepackage{wrapfig,lipsum,booktabs}

\usepackage{multicol,lipsum}
\usepackage{multirow}
\usepackage{hhline}

\usepackage{amsmath}
\usepackage{amsfonts}
\usepackage{amsthm}
\usepackage{amssymb}
\usepackage{mathrsfs}
\usepackage{bbm}

\usepackage{algorithm}
\usepackage{algpseudocode}

\usepackage{amsthm}

\newtheorem{theorem}{Theorem}

\newtheorem{lemma1}[theorem]{Lemma}

\newtheorem{definition}[theorem]{Definition}
\newtheorem{proposition}[theorem]{Proposition}
\newtheorem{remark}{Remark}[theorem]

\usepackage{caption} 
\newcommand\corrauthor[1]{%
  \begingroup
  \renewcommand\thefootnote{}\footnote{#1}%
  \addtocounter{footnote}{-1}%
  \endgroup
}

\begin{document}

\begin{center}
\begin{spacing}{1.5}
\textbf{\Large The Joint Gromov Wasserstein Objective for Multiple Object Matching}
\end{spacing}

\vspace{5mm}
Aryan Tajmir Riahi$^{1}$, and Khanh Dao Duc$^{1,2,*}$ \corrauthor{$^*$ corresponding author: kdd@math.ubc.ca}

\vspace{5mm}
$^{1}$ Department of Computer Science, University of British Columbia, Vancouver, BC V6T 1Z4, Canada\\
$^{2}$ Department of Mathematics, University of British Columbia, Vancouver, BC V6T 1Z4, Canada\\

\vspace{5mm}
\vspace{5mm}
\end{center}

\begin{abstract}
  The Gromov-Wasserstein (GW) distance serves as a powerful tool for matching objects in metric spaces. However, its traditional formulation is constrained to pairwise matching between single objects, limiting its utility in scenarios and applications requiring multiple-to-one or multiple-to-multiple object matching. In this paper, we introduce the Joint Gromov-Wasserstein (JGW) objective and extend the original framework of GW to enable simultaneous matching between collections of objects. Our formulation provides a non-negative dissimilarity measure that identifies partially isomorphic distributions of mm-spaces, with point sampling convergence. We also show that the objective can be formulated and solved for point cloud representations by adapting traditional algorithms in Optimal Transport, including entropic regularization. Our benchmarking with other variants of GW for partial matching indicates superior performance in accuracy and computational efficiency of our method, while experiments on both synthetic and real-world datasets show its effectiveness for multiple shape matching, including geometric shapes and biomolecular complexes, suggesting promising applications for solving complex matching problems across diverse domains, including computer graphics and atomic model building for structural biology.

\end{abstract}

\section*{Introduction}\label{sec:intro}
Finding correspondence between two objects is a central problem in computer science with applications such as shape interpolation and texture transfer in computer vision~\cite{sahilliouglu2020recent}, account linking across social networks~\cite{zhang2015multiple}, or protein structure analysis  ~\cite{riahi2023alignot, riahi2025alignment}. Despite extensive research on object matching across these domains, existing methods mostly address full-to-full shape matching~\cite{sahilliouglu2020recent, roetzer2024spidermatch, ovsjanikov2012functional, cao2023unsupervised}, where both objects are assumed to be complete with no significant missing parts. A smaller body of work focuses on partial-to-full matching~\cite{sahilliouglu2020recent, ehm2024geometrically, cosmo2016shrec, litany2017deep}, where an incomplete object is matched against a complete reference. However, when multiple fragments must be assembled—a scenario arising in protein model building~\cite{riahi2025alignment}, merging partial 3D scans~\cite{neugebauer1997reconstruction}, and solving 2D and 3D puzzles~\cite{domokos2010affine,litany2020non}— current approaches require sequential pairwise matching. This strategy can lead to error accumulation and increased computational cost, motivating the need for multiple-to-multiple partial matching methods.

Recently, Optimal Transport theory (OT) \cite{peyre2019computational} and its related tools became a popular choice establishing correspondence between two objects represented as measurements over a metric space in graph matching \cite{xu2019gromov,li2023convergent}, graph clustering \cite{chowdhury2021generalized}, matching language models \cite{grave2019unsupervised}, and biomolecule matching \cite{riahi2025alignment, riahi2023alignot, singer2024alignment}. In particular, approximations of the Gromov-Wasserstein (GW) distance \cite{memoli2011gromov} have gained popularity, as they provide a powerful tool for matching objects defined in different domains and are rigid body transformation invariant (so they don't require pre-alignment). Furthermore, researchers have designed various extensions and variants of the Gromov-Wasserstein distance \cite{bai2024efficient, chapel2020partial, sejourne2021unbalanced} for partial-to-full matching. The Z-Gromov-Wasserstein distance~\cite{bauer2024z} was recently introduced for the multiple-to-multiple partial matching problem, assuming a Z-structure on the distributions. However, many multiple-to-multiple matching applications lack such structure in the data.


In this paper, we introduce the Joint Gromov-Wasserstein (JGW) problem, that is a novel variant of the original Gromov-Wasserstein formulation that enables the matching of two collections of objects simultaneously.
To summarize, our key contributions are as follows.

\begin{itemize}

\item We formulate the new JGW objective function, which extends the mathematical concepts underlying the Gromov-Wasserstein objective function, such as metric measure spaces and isomorphisms, to handle collections of objects. We prove theoretical properties of the JGW objective, analyzing its metric properties and convergence from point sampling.
    \item We investigate and adapt existing approximation techniques from the standard Gromov-Wasserstein problem to our the JGW framework, to produce feasible and applicable algorithms to compute it.
    \item We demonstrate the usefulness of the JGW formulation through benchmarks against other GW variants, and experiments on object matching problems with various datasets including 2D/3D geometric shapes and biomolecular complexes.

    \end{itemize}

\section*{Related Work}
  Here, we cover most recent works on partial and multiple-to-multiple object matching, with a focus on OT-related approaches. We refer interested readers to  \cite{sahilliouglu2020recent} for a broader survey of object matching methods. 


\paragraph{Partial Matching.} Partial-to-full matching, where an incomplete query must be aligned to a complete template, has received considerable attention due to its practical importance in object recognition and retrieval. Early approaches adapted full matching techniques by incorporating outlier handling mechanisms~\cite{sahilliouglu2020recent} or by identifying and matching salient regions~\cite{rodola2017partial}. The SHREC benchmark for partial matching~\cite{cosmo2016shrec} has driven progress in this area, with top-performing methods leveraging learned descriptors~\cite{litany2017deep}, and region growing strategies~\cite{rodola2017partial}. Graph matching has been extended to the partial setting through modifications that allow node and edge deletions~\cite{cho2013learning, zanfir2018deep}, while point cloud methods have incorporated robust estimators~\cite{yang2020teaser} and learned features~\cite{huang2021predator} to handle missing data. In the optimal transport framework, partial variants such as unbalanced~\cite{chizat2018unbalanced, sejourne2021unbalanced} and semi-relaxed~\cite{chapel2020partial, bai2024efficient} optimal transport have been developed.

\paragraph{Multiple-to-Multiple Matching.} Matching multiple objects simultaneously, rather than through sequential pairwise alignments, has been explored primarily in the context of full shape collections. Litany et al.~\cite{litany2020non} used an extension of Partial Functional Maps to introduce a framework for multiple-to-multiple shape matching. Wu et al. introduced an alternative approach by simultaneous partial functional
 correspondence \cite{wu2023multi}. In the context of Optimal Transport Theory, the Z-Gromov-Wasserstein distance~\cite{bauer2024z} was recently introduced, extending the traditional Gromov-Wasserstein framework to match distributions equipped with Z-structure. However, many multiple-to-multiple matching applications lack such structure in the data.



\section*{Gromov-Wasserstein Distance}
\label{sec:gw}
In this section we briefly introduce the traditional Gromov-Wasserstein distance \cite{memoli2011gromov} along with key definitions.

\subsection*{Preliminaries}
Suppose we are given two compact metric spaces $(X, d_X), (Y, d_Y)$ and measures $\mu_X, \mu_Y$. Following \cite{memoli2011gromov}, define the \textit{metric measure space} (\textit{mm-space}) and the \textit{set of all couplings}, two fundamental concepts for the definition of the Gromov-Wasserstein distance, as follows.

\begin{definition}\label{def:mmspace}
A \textit{metric measure space} (mm-space) is a triple $(X,d_X,\mu_X)$, where $(X,d_X)$ is a compact metric space and $\mu_X$ is a Borel probability measure, i.e., $\mu_X(X) = 1$ and $\text{supp}[\mu_X]=X$. 
An example of a discrete mm-space is illustrated in Figure \ref{fig:example1}\textbf{a}.
\end{definition}


\begin{definition}\label{def:isomorphism}
Two mm-spaces $(X, d_X, \mu_X)$ and $(Y,d_Y,\mu_Y)$ are called \textit{isomorphic} if there exists isometry $\psi: X \rightarrow Y$, i.e., $d_x(x,x') = d_y(\psi(x), \psi(x'))$ for any $x,x' \in X$, such that $(\psi\#\mu_X) = \mu_Y$, where $\#$ denotes the pushforward operator. Note that isomorphism is an equivalence relation, and the GW problem aims to define a metric between equivalence classes of mm-spaces.
\end{definition}

The main goal of the GW problem is to define a metric between ``nonequal" classes of mm-spaces. To complete this task, \cite{memoli2011gromov} defines \textit{isomorphism} as a notion of equality between mm-spaces.

\begin{definition}
    Given two mm-spaces $(X, d_X, \mu_X)$ and $(Y,d_Y,\mu_Y)$, $\mathcal{M}(\mu_X, \mu_Y)$ denotes the set of all transportation plans, such that, $\mu \in \mathcal{M}(\mu_X, \mu_Y)$ is a Borel probability measure on $X \times Y$, and satisfies the marginal constraints $\mu(A \times Y) = \mu_x(A)$ for any Borel subset $A \subset X$, and $\mu(X \times B) = \mu_y(B)$ for any Borel subset $B \subset Y$.
\end{definition}

\subsection*{Formulation}

These definitions enable us to define the \textit{Gromov-Wasserstein distance} as a comparison method between mm-spaces.

\begin{definition}\cite{memoli2011gromov}
Given two mm-spaces $(X, d_X, \mu_X)$ and $(Y,d_Y,\mu_Y)$, the \textit{Gromov-Wasserstein distance} between $X$ and $Y$ is defined as
\begin{equation} \label{eq:gw-definition}\begin{aligned}
&\mathcal{GW}_{\Gamma,p}(X,Y)=\\&  \underset{\mu \in \mathcal{M}(\mu_X, \mu_Y)}{\text{inf}} \frac{1}{2} \left( \int_{X\times Y}\int_{X\times Y} \Gamma(x,y,x',y')\mu(dx\times dy)\mu(dx'\times dy')\right)^{1/p},\end{aligned}\end{equation}
where $\Gamma: X \times Y \times X \times Y \rightarrow \mathbb{R}$ is called the \textit{loss function}. With the typical choice of $\Gamma_p(x,y,x',y') = |d_X(x,x') - d_Y(y,y')|^p$ we often denote $\mathcal{GW}_{\Gamma_p, p}(X,Y)$ by $\mathcal{GW}_{p}(X,Y)$. 
\end{definition}


The minimizer of this optimization problem is called the \textit{transportation plan}, and it can be used to find a matching between $X$ and $Y$, as a metric function between isomorphy classes of mm-spaces:


\begin{theorem}[\cite{memoli2011gromov}]\label{thm:isomorphism-memoli}
    $\mathcal{GW}_p$ defines a metric on the collection of all isomorphism classes of mm-spaces.
\end{theorem}

\section*{The Joint Gromov-Wasserstein Objective}\label{sec:jgw}
\subsection*{Preliminaries and Definition} 
To enable multiple-to-multiple object matching, we extend fundamental concepts associated with the GW distance, by first introducing \textit{distributions of metric measure spaces} (see also Definition \ref{def:mmspace}).

\begin{definition}
A \textit{distribution of mm-spaces} is a categorical distribution of $k_X$ mm-spaces, usually denoted  $\mathbf{X} = (X_i, d_{X_i}, \mu_{X_i}, s_{X_i})_{i\in[k_X]}$, where

$(i)$ $\forall i \in [k_X]$, $(X_i, d_{X_i}, \mu_{X_i})$ is a metric measure space (called \textit{cluster $i$})
 
 $(ii)$ $s_{X_i} \in \mathbb{R}_{>0}$ is the probability assigned to cluster $i$.

\end{definition}


Figure \ref{fig:example1} shows an illustrative comparison between a mm-space (\ref{fig:example1}\textbf{a}) and a distribution of mm-spaces (\ref{fig:example1}\textbf{b}). To provide a framework for comparing distributions of mm-spaces, we introduce the notion of \textit{embedding}:

\begin{definition}\label{def:embedding}
Given a distribution of mm-spaces  $\mathbf{X} = (X_i, d_{X_i}, \mu_{X_i}, s_{X_i})_{i\in [k_X]}$, an \textit{embedding} of $\mathbf{X}$ is a mm-space $(X, d_X, \mu_X)$ such that 
there exist $k_X$ isometries $(\psi_i: X_i \rightarrow X)_{i\in[k_X]}$, such that 

$(i)$ $\sum_{j\in[k_X]} s_{X_j}\times \psi_j\#\mu_{X_j} = \mu_X$ 

$(ii)$ $\forall (j,\ne k) \in [k_X]^2$, $\psi_j(X_j) \cap \psi_k(X_k) = \emptyset$ 

$(iii)$ $X = \bigcup_{i \in [k_X]} \psi_{X_i}(X_i).$

We call the $\psi_i$'s \textit{embedding function}s. 
\end{definition}

Using embeddings, we now formulate the Joint Gromov-Wasserstein objective:

\begin{definition}
Given two distributions of mm-spaces  $\mathbf{X}$ and $\mathbf{Y}$
and embeddings $(X, d_X, \mu_X)$ and $(Y,d_Y,\mu_Y)$ with embedding functions $(\psi_{X_i})_{i\in[k_X]}$ and $(\psi_{Y_i})_{i\in[k_Y]}$ respectively, the \textit{joint Gormov-Wasserstein divergence}
between $\mathbf{X}$ and $\mathbf{Y}$ is defined by
\begin{equation}
\label{eq:jgw-definition}
\mathcal{JGW}_p(\mathbf{X},\mathbf{Y}) = \mathcal{GW}_{\Gamma_p^*,p}(X,Y),
\end{equation}
where for all $(i,j) \in [k_X]\times[k_Y]$ and $(x,x',y,y')\in Im(\psi_{X_i})^2\times Im(\psi_{Y_j})^2$,
\begin{equation}\label{eq:def-gammastar}
    \Gamma_p^*(x,y,x',y') = |d_X(x,x') - d_Y(y,y')|^p,
\end{equation}
and $\Gamma_p^* =0$ otherwise. 
\end{definition}

 We note that while the definition of $\mathcal{JGW}_p(\mathbf{X},\mathbf{Y})$ uses given embeddings $X,Y$, its value and the associated transport plan do not depend on these. They also neither depend on the choice of the embedding functions:
 
\begin{theorem}\label{thm:embedding-unique}
    Given two distributions of mm-spaces  $\mathbf{X}$ and $\mathbf{Y}$ and different embeddings $X_1, X_2$ for $\mathbf{X}$ and $Y_1, Y_2$ for $\mathbf{Y}$, 
    we have
    $$\mathcal{GW}_{\Gamma_p^*,p}(X_1,Y_1) = \mathcal{GW}_{\Gamma_p^*,p}(X_2,Y_2).$$
\end{theorem}

The proof, detailed in Appendix \ref{app:uniqueness} constructs a cost-preserving bijection between $\mathcal{M}(\mu_{X_1}, \mu_{Y_1})$ and $\mathcal{M}(\mu_{X_2}, \mu_{Y_2})$ with respect to the $\Gamma^*_p$ cost.

\subsection*{Partial Ismorphism}

To establish key properties of the Joint Gromov Wassertein objective function as a similarity measure between distributions of mm-spaces, we now extend the notion of isomorphism (Definition \ref{def:isomorphism}) to \textit{partial isomorphism} as follows. 

\begin{definition}
\label{def:partialisomorphism}
    Two distributions of mm-spaces  $\mathbf{X} = (X_i, d_{X_i}, \mu_{X_i}, s_{X_i})_{i\in[k_X]}$ and $\mathbf{Y} = (Y_i, d_{Y_i}, \mu_{Y_i}, s_{Y_i})_{i\in[k_Y]}$ are called \textit{partially isomorphic} 
    if 
    there exists a distribution of mm-spaces $\mathbf{Z}$, indexed by $(i,j)\in [k_X \times k_Y]$
    $$\mathbf{Z} = (Z_{i,j}, d_{Z_{i,j}}, \mu_{Z_{i,j}}, s_{Z_{i,j}})_{i\in[ k_X], j\in [k_Y]},$$ and 
    isometry functions $\psi^X_{i,j}:Z_{i,j} \rightarrow X_i$ and $\psi^Y_{i,j}: Z_{i,j} \rightarrow Y_{j}$, such that $$\sum_j \psi^X_{i,j}\#\mu_{Z_{i,j}} \times s_{Z_{i,j}} = \mu_{X_{i}}\times s_{X_i} \text{ and } \sum_i\psi^Y_{i,j}\#\mu_{Z_{i,j}} \times s_{Z_{i,j}} = \mu_{Y_{j}}\times s_{Y_j}.$$ 
\end{definition}

Similar to \cite{memoli2011gromov} for mm-spaces, we can then extend Theorem \ref{thm:isomorphism-memoli} to distributions of mm-spaces:

\begin{theorem}\label{thm:isomorphism}
Given two distribution of mm-spaces $\mathbf{X}$, $\mathbf{Y}$ and $p \in [1, \infty)$, $\mathcal{JGW}_p(\mathbf{X}, \mathbf{Y}) = 0$ if and only if $\mathbf{X}$ and $\mathbf{Y}$ are partially isomorphic.
\end{theorem}

The proof, detailed in Appendix \ref{app:proofs} establishes the existence of a transport map minimizing the JGW cost, then leverages this map to construct a partial isomorphism between the two distributions.

\begin{remark}
    Although Theorem \ref{thm:isomorphism} shows that some properties of isomorphism of $\mathcal{GW}_p$ naturally extend to $\mathcal{JGW}_p$, note that the Joint Gromov-Wasserstein objective function does not form a proper distance function that holds the triangle inequality (which is not issue for our goal of matching two collections of objects). As a counterexample, let $\mathbf{X} = \{0, 1\}, \mathbf{Z} = \{0, 2\}$, both equipped with the uniform distribution and $\mathbf{Y}$ be a distribution of mm-spaces with two one-point set clusters of equal mass. One can verify that $\mathbf{X}$ and $\mathbf{Y}$ are partially isomorphic, thus, $\mathcal{JGW}_p(\mathbf{X}, \mathbf{Y}) = 0$. With the same argument we can see that $\mathcal{JGW}_p(\mathbf{Z}, \mathbf{Y}) = 0$. However, $\mathbf{X}$ and $\mathbf{Z}$ are not partially isomorphic hence $\mathcal{JGW}_p(\mathbf{Z}, \mathbf{X}) > 0$.  

\end{remark}

\subsection*{Point Sampling Convergence}

In the context of shape matching, having point sampling convergence for the objective is crucial, since objects get discretized or represented by point clouds.  The following theorem ensures that we also asymptotically recover JGW when doing so: 

\begin{theorem}\label{thm:sampling} 
    Let $\mathbf{X} = (X_i, d_{X_i}, \mu_{X_i}, s_{X_i})_{i\in[k_X]}$ be a distribution of mm-spaces and $p\in [1, \infty)$, $n \in \mathbb{N}$. Consider $n$ i.i.d samples from $\mathbf{X}$ (by randomly picking a cluster $j$ from Cat($s_i$) and sampling a point in $X_j$ from $\mu_{X_j}$), distributed into the $k_X$ mm-spaces of $\mathbf{X}$ as $\{X^n_i\}_{i \in [k_X]}$. Let $\mathbf{X}^n$ be a distribution of mm-spaces defined as $(X^n_i, d_{X_i}, \mu_i, s_{X_i})_{i \in [k_X]}$ where $\mu_i$ is the uniform measure on $X^n_i$. Then $\mathcal{JGW}_p(\mathbf{X}^n, \mathbf{X}) \rightarrow 0$ almost surely as $n \rightarrow \infty$.
\end{theorem}

To prove the theorem (details in Appendix \ref{app:proof-samp}), we establish an inequality between JGW cost and the Gromov-Wasserstein distance that yields the result.

\subsection*{The Joint Gromov-Wasserstein Objective in Finite Space}
In practice, we are interested in solving a discretized version of the Joint Gromov-Wasserstein objective function. Formally, let $\mathbf{X} = (X_i, d_{X_i}, \mu_{X_i}, s_{X_i})_{i\in[k_X]}$ and $\mathbf{Y} = (Y_i, d_{Y_i}, \mu_{Y_i}, s_{Y_i})_{i=1\in[k_Y]}$, with $X_i$ and $Y_i$ being finite for all $i$, and let us denote $n_{X_i}$ and $n_{Y_i}$ the cardinal of $X_i$ and $Y_i$ respectively, so  $X_i = \{x_{i,j}\}_{j\in[n_{X_i}]}$ and $Y_i = \{y_{i.j}\}_{j=1\in n_{Y_i}}$, with pairwise distance matrices $d_{X_i} \in \mathbb{R}_{\ge 0}^{n_{X_i} \times n_{X_i}}$ and $d_{Y_i} \in \mathbb{R}_{\ge 0}^{n_{Y_i} \times n_{Y_i}}$. To simplify our embedding notation, we also denote $X = \bigcup_{i=1}^{k_X} X_i$ and $Y = \bigcup_{i=1}^{k_Y} Y_i$. We now create $\mu_X$ and $\mu_Y$ as distributions over $X$ and $Y$ respectively, as
\begin{equation}
    \nonumber
    \mu_X[x_{i,j}] = s_{X_i}\mu_{X_i}[x_{i,j}], \quad\mu_Y[y_{i,j}] = s_{Y_i}\mu_{Y_i}[y_{i,j}].
\end{equation}
Note that with this definition we clearly have $\mu_X[X] = \mu_Y[Y] = 1$.
Next, we define the block matrices
\begin{align}
d^X = \begin{bmatrix}
    d_{X_1} & 0 & \dots & 0\\
    0 & d_{X_2} & \dots & 0\\
    \vdots & \vdots &\ddots & \vdots\\
    0 & 0 & \dots & d_{X_{k_X}}
\end{bmatrix},\quad
I^X = \begin{bmatrix}
    J_{n_{X_1}, n_{X_1}} & 0 & \dots & 0\\
    0 & J_{n_{X_2}, n_{X_2}} & \dots & 0\\
    \vdots & \vdots &\ddots & \vdots\\
    0 & 0 & \dots & J_{n_{X_{k_X}}, n_{X_{k_X}}}
\end{bmatrix}. \nonumber
\end{align}

 With these definitions, we have $d^X \otimes I^X = d^X$, and $d^Y \otimes I^Y = d^Y$ where $\otimes$, denotes the elementwise multiplication and $j_{n\times m}$ denotes a $n \times m$ where all entries all one. For simplicity, we now assume $p=2$ and denote $\mathcal{JGW}_2(\mathbf{X}, \mathbf{Y})$ by  $\mathcal{JGW}(\mathbf{X}, \mathbf{Y})$. Using \eqref{eq:gw-definition} and \eqref{eq:jgw-definition}, we have
\begin{align} \label{eq:jgw-discrete}
    \mathcal{JGW}(\mathbf{X}, \mathbf{Y}) =  \underset{\mu \in \mathcal{M}(\mu_X, \mu_Y)}{\text{min}} \frac{1}{2} \left(\sum_{i,j,k,l} \lvert d^X_{ij} - d^Y_{kl} \rvert^2 I^X_{ij}I^Y_{kl} \mu_{ik}\mu_{jl} \right)^{1/2}.
\end{align}

\subsection*{Regularization and Computation}
Our goal is now to compute $\mathcal{JGW}(\mathbf{X}, \mathbf{Y})$ as given in equation \eqref{eq:jgw-discrete}. One of the main limitations of the Gromov-Wasserstein distance is the non-convexity of its formulation which makes its computation challenging. To overcome this challenge, various approximations and algorithms for GW distance or its variants have been proposed, such as linear lower bounds \cite{memoli2011gromov}, entropic regularization \cite{solomon2016entropic},  operator splitting-based relaxation \cite{li2023convergent}, and Frank-Wolfe optimization
algorithm \cite{chapel2020partial}. To compute \eqref{eq:jgw-discrete} we can adapt most of these techniques, including methods used in \cite{li2023convergent, solomon2016entropic, chapel2020partial}, and one of the linear lower bounds proved in \cite{memoli2011gromov}. In the rest of this paper, as a proof of concept, we focus on adapting entropic regularization \cite{solomon2016entropic}, which is one of the most widely used approximations for OT problems \cite{cuturi2013sinkhorn}. To do so, we introduce the regularization term to \eqref{eq:jgw-discrete}:
\begin{align} \label{eq:jgw-entropic}
    \mathcal{JGW}^\epsilon(\mathbf{X}, \mathbf{Y}) = \underset{\mu \in \mathcal{M}(\mu_X, \mu_Y)}{\text{min}} \frac{1}{2}\left(\sum_{i,j,k,l} \lvert d^X_{ij} - d^Y_{kl} \rvert^2 I^X_{ij}I^Y_{kl} \mu_{ik}\mu_{jl}  - \epsilon H(\mu)\right)^{1/2}, \nonumber
\end{align} 
where $H$ is the usual entropy function defined by
\begin{equation}\label{eq:entropy-definition}
    H(\mu) = -\sum_{i,j}\mu_{i,j}\log (\mu_{i,j}),\nonumber
\end{equation}
and $\epsilon \in \mathbb{R}_{\ge 0}$ is called the \emph{regularization parameter}.

\begin{proposition}\label{prp:regularization}
    Given $\mathbf{X}, \mathbf{Y}$ and $\epsilon \in \mathbb{R}_{\ge 0}$, we have

    \begin{equation}\label{eq:jgw-entropic-simplified}
    \mathcal{JGW}^\epsilon(\mathbf{X}, \mathbf{Y}) = \underset{\mu \in \mathcal{M}(\mu_X, \mu_Y)}{\text{min}} \frac{1}{2} \left( \langle \mu, \Lambda(\mu) \rangle  - \epsilon H(\mu)\right)^{1/2},\nonumber
\end{equation}
where $\langle .,.\rangle$ is the inner product of two given matrices, the superscript $\wedge2$ denotes the elementwise square of a matrix, and $\Lambda$ is defined as
\begin{equation}
    \Lambda(\mu) = d^{X\wedge2} \mu I^Y  - 2  d^X \mu d^Y + I^X \mu d^{Y\wedge2}.\nonumber
\end{equation}
\end{proposition}

We prove Proposition \ref{prp:regularization} in Appendix \ref{app:approx} by simply expanding the equation. 
Using this Theorem and with the same argument as \cite{benamou2015iterative, solomon2016entropic}, we can compute the transportation plan ($\mu$) by solving
\begin{equation}\label{eq:kl-equiv-problem}
    \mu = \underset{\mu \in \mathcal{M}(\mu_X, \mu_Y)}{\text{argmin}} \text{KL} (\mu, e^{\frac{-\Lambda(\mu)}{\epsilon}}).
\end{equation}
To solve \eqref{eq:kl-equiv-problem}, similar to \cite{benamou2015iterative, solomon2016entropic} we can use the following iterations
\begin{equation}\label{eq:kl-iterations}
    \mu^{(t+1)} = \underset{\mu \in \mathcal{M}(\mu_X, \mu_Y)}{\text{argmin}} \text{KL} (\mu, \left[e^{\frac{-\Lambda(\mu^{(t)})}{\epsilon}}\right]^{\wedge \eta} \otimes \left[ \mu^{(t)} \right]^{\wedge(1-\eta)}),
\end{equation}
where $0 < \eta \le 1$ is called the convergence parameter. Note that the number of iterations of \eqref{eq:kl-iterations} needed for convergence increases as $\eta \rightarrow 0$, but the iterations might not converge for high enough values of $\eta$. 
Pseudocode for this method is provided in Algorithm \ref{alg:main-method}. As this algorithm performs a sinkhorn projection in each iteration, using \cite[Remark 4.6]{peyre2019computational}, the overall time complexity is $\mathcal{O}(tn^2\log(n)\epsilon^{-3})$,  where $n$ is the number of points, $t$ is the number of iterations, and $\epsilon$ is the regularization parameter.

\begin{algorithm}
    \caption{Pseudocode for approximating JGW solver }
    \label{alg:main-method}
        \hspace*{\algorithmicindent} \textbf{Input} convergence parameter $\eta \in (0, 1]$, regularization parameter $\epsilon \in \mathbb{R}_{\ge 0}$, maximum number of iterations $T \in \mathbb{N}$ and two embedded distribution of mm-spaces $\mathbf{X}, \mathbf{Y}$ given by distance matrices $d^X \in \mathbb{R}^{n_X \times n_X}, d^Y \in \mathbb{R}^{n_Y \times n_Y}$, cluster matrices $I^X \in \mathbb{R}^{n_X \times n_X}, I^Y \in \mathbb{R}^{n_Y \times n_Y}$, and marginal vectors $\mu_{X} \in \mathbb{R}^{n_X}, \mu_{Y} \in \mathbb{R}^{n_Y}$
    \begin{algorithmic}[1]
        \State $\mu := \mu_X \times \mu_Y^T$
        \For {$t$ in $1,2,\dots,T$}
        \State $\Lambda = d^{X\wedge2} \mu I^Y  - 2  d^X \mu d^Y + I^X \mu d^{Y\wedge2}$
        \State $K = \left[e^{\frac{-\Lambda}{\epsilon}}\right]^{\wedge \eta} \otimes \left[ \mu \right]^{\wedge(1-\eta)}$
        \State $\mu = \text{SINKHORN-PROJECTION}(K, \mu_X, \mu_Y)$
        \EndFor
        \State \Return $\mu$
    \end{algorithmic}
\end{algorithm}

\section*{Experiments and Results}

For all the following experiments, the computational times reported were obtained on the same machine with a 12th Gen Intel(R) Core(TM) i5-1240P (1.70 GHz) CPU and 16.0 GB of RAM. None of the following experiments uses a GPU. All of the code for this paper is implemented in Python 3.10. and is available in this \href{https://github.com/artajmir3/Joint-Gromov-Wasserstein}{repository}.

\subsection*{General Evaluation on Partial Matching}\label{sec:JGW-partial} 
To evaluate the performance of the JGW objective, we first conducted a partial matching experiment by comparing JGW with other recent Gromov-Wasserstein variants involving two unbalanced measurements $p,q \quad (1 = |p| > |q|)$, including \emph{the mass-constrained Partial Gromov-Wasserstein distance} (mPGW) \cite{chapel2020partial}, \emph{the Partial Gromov-Wasserstein distance} (PGW) \cite{bai2024efficient}, \emph{and the Unbalanced Gromov-Wasserstein distance} (UGW) \cite{sejourne2021unbalanced}. Although JGW originally operates on data from all clusters, we adapted it for this specific scenario by employing dummy clusters through the construction $\mathbf{Y} = \{p\}$ and $\mathbf{X} = \{X_1, X_2\}$ where $X_1$ is $q/|q|$ with mass $s_{X_1} = |q|$ and $X_2$ is a single point distribution with mass $s_{X_2} = 1 - |q|$. Using the formulation of \eqref{eq:jgw-definition}, we notice that the formulation of mPGW \cite{chapel2020partial} is equivalent to this special case of JGW. However, the computation method suggested by its authors differs from the approximation we used for our computations.

We generated the source distribution ($q$) by sampling $200$ points from an Archimedean spiral with added noise, while the target distribution ($p$) combines 200 points sampled using the same method, but with 100 additional points drawn from a standard normal distribution (see Figure \ref{fig:jgw-partial-coupling}\textbf{a}). We then applied mPGW\cite{chapel2020partial}, PGW\cite{bai2024efficient}, UGW\cite{sejourne2021unbalanced}, and our proposed JGW method to compute the optimal coupling between $p$ and $q$, with results visualized in Figure \ref{fig:jgw-partial-coupling}\textbf{b}. The results reveal a significant performance difference in structural preservation. Both mPGW and PGW exhibit substantial difficulty in distinguishing spiral structure from added noise, incorrectly transporting nearly half ($49\%$) of the total mass to noise points. UGW demonstrates some improvement over these methods but still suffers from significant misattribution, transporting $33\%$ of mass toward noise points, and failing to fully capture the spiral's structural coherence. In contrast, JGW achieves superior performance by effectively separating the true spiral structure from noise contamination, transporting only $0.9\%$ of mass to noise points while preserving the geometric integrity of the spiral pattern. Although we acknowledge that mPGW and PGW results depend critically on the initial coupling of their algorithms, our search across $50$ randomly generated transportation plans, combined with the authors' default initialization strategy, failed to yield satisfactory couplings, suggesting fundamental limitations in these methods' ability to handle partial matching tasks.

\subsection*{Sparcity of Transport Plans}
A critical consideration in shape-matching and other continuous-data applications is the quality of the coupling matrix, as accurate one-to-one correspondences for geometric analysis require sparse transportation matrices. Since mPGW and PGW naturally yield one-to-one matchings, we focus here on UGW and JGW, which produce diffuse transportation patterns due to their regularization terms. To examine this, we revisited the previous experiment and visualized the transportation edges from the leftmost point in $q$ (shown in purple in Figure \ref{fig:jgw-partial-quality}\textbf{a}), revealing that JGW achieves slightly better sparsity than UGW. To quantify this difference, we computed the variance of transported mass from each source point in $q$ for both methods.
Since both sparsity and runtime are heavily influenced by the regularization parameter, we reported the runtime and average per-point variance across regularization values ranging from $0.1$ to $1$ for JGW and $0.005$ to $0.2$ for UGW, with the resulting trade-off shown in Figure \ref{fig:jgw-partial-quality}\textbf{b}. JGW yields approximately a $25\times$ speedup over UGW at comparable variance levels, and roughly a $6\times$ improvement in variance at comparable runtimes. To verify robustness, Figure \ref{fig:jgw-partial-quality}\textbf{c} reports the matching error for each method under the regularization parameters used previously. Regarding runtime, mPGW and PGW complete in $0.49$ and $0.48$ seconds, respectively, benefiting from their direct optimization approach, while the regularized methods require substantially more time.
In summary, Figures \ref{fig:jgw-partial-coupling} and \ref{fig:jgw-partial-quality} demonstrate that JGW produces couplings most consistent with structural expectations, outperforming all competing methods in preserving meaningful geometric correspondences and achieving a superior sparsity-runtime trade-off compared to UGW. While mPGW and PGW offer computational efficiency and one-to-one couplings, these advantages come at the high cost of failing to preserve the underlying geometric structure.

\subsection*{Applications of JGW in Shape Matching}
To analyze the performance of JGW in applications related to shape matching, we designed experiments involving 2D and 3D shape data. First, we used a typeset illustrating of three letters ``A", ``B" and ``C" (see Figure \ref{fig:jgw-shape} \textbf{a} source) to build a distribution of mm-spaces with 3 clusters and use as the source space. For the target space, we used a different typeface and created one cluster with the illustration of the word ``ABC" (see Figure \ref{fig:jgw-shape} \textbf{a} target). 
In Figure \ref{fig:jgw-shape} \textbf{a}, we color in the right panel "result" each point in the target distribution based on the cluster from the source space, that has its corresponding coupled point in the JGW transportation plan. JGW manages to transport 98.6$\%$ of the mass correctly, by using $450$ points to represent the data. This experiment took $\approx24$ seconds to run.

To test the performance of JGW on 3D data, we next used two 3D meshes from the CAPOD dataset \cite{papadakis2014canonically} of a human in different poses. We split the first mesh into 5 clusters, namely, upper body, left arm, right arm, left leg, and right leg (see Figure \ref{fig:jgw-shape} \textbf{b}). We used the mesh vertices as our distribution points and applied the same method to find correspondence and colorized matched points of each cluster in one color. As Figure \ref{fig:jgw-shape} \textbf{b} suggests, JGW does a perfect job in distinguishing the arms and the body, although it confuses some parts of the legs. In total more than $80\%$ of the mass is transported correctly, and it took $50$ seconds to perform this experiment. We acknowledge that due to symmetry in this particular example, it's possible to get the same result with the substitution of left and right arm/leg, depending on the initialization of our optimization process. 

Finally, 
we evaluated our method on the SHREC'16 cuts dataset~\cite{cosmo2016shrec}, which uses shapes from TOSCA~\cite{bronstein2008numerical} and provides partial versions with different cuts (see Figure \ref{fig:jgw-shape} \textbf{c}). Since the SHREC'16 dataset provides only one partial cluster per shape, we computed the complement component as the second cluster to make it suitable for our method. Figure \ref{fig:jgw-shape} \textbf{c} shows the colorized correspondence diagram, demonstrating near-perfect mass transportation.
To quantitatively evaluate the mapping quality, we employed a standard measure commonly used in shape matching~\cite{kim2011blended, wu2023multi, ehm2024geometrically}, that is the geodesic distance between ground truth and computed corresponding points, normalized by the square root of the full shape's area. Figure \ref{fig:jgw-shape} \textbf{c} (Correspondence quality) presents a cumulative distribution function (CDF) of this measure across all mesh vertices. The fact that the CDF reaches $100\%$ at a geodesic error of $0.0001$ suggests that this method can match complex shapes with excellent accuracy. For this experiment we modeled the cat body with $10^4$ points and it took $\approx3000$ seconds to run this experiment.

Overall, these experiments show the potential of JGW in shape matching problems.


\subsection*{Alignment of Biomolecular Complexes}\label{sec:JGW-cryo} 
As mentioned in the Introduction, one of the potential applications of our method is the alignment and fitting of biomolecules from 3D density maps obtained from Cryogenic Electron Microscopy (cryo-EM) \cite{riahi2023alignot,riahi2025alignment}. 
To investigate how our framework performs in this context, we focus on the model-building task, in which a density map (a large 3D voxelized array) and the atomic structures of its submodule are given, and the goal is to find the optimal fitting of each submodule within the full map \cite{riahi2025alignment}. This setting entirely matches our problem formulation and assumptions. We compare our proposed joint alignment method against EMPOT \cite{riahi2025alignment}, which performs one-by-one submodule alignment using Unbalanced Gromov–Wasserstein (UGW)\cite{sejourne2021unbalanced}. Experiments are conducted on a dataset of 43 atomic models obtained from \cite{he2022model}.
For each model, we partition the structure into its constituent chains, with the number of chains ranging from 2 to 10. Figure \ref{fig:bio}\textbf{a} illustrates one such model (PDB:1I3Q \cite{cramer2001structural}), with each chain rendered in a distinct color. Each model is converted into a density map, from which 1,000 points are sampled in total. We then apply both UGW and JGW to the resulting point clouds. Figure \ref{fig:bio}\textbf{b} displays, for PDB:1I3Q, the aligned position of each chain produced by each method (shown in red) alongside its ground-truth position (shown in blue), as well as the full reconstructed model. As the figure demonstrates, JGW achieves a near-perfect alignment across all chains, whereas UGW produces substantially incorrect alignments. We perform this experiment across the full dataset and evaluate performance using the root mean square deviation (RMSD) of the atomic models after applying each method's alignment. Figure \ref{fig:bio}\textbf{c} presents the results as a scatter plot, suggesting a significant performance improvement of JGW over UGW across the whole dataset.
Out of 43 models, JGW achieves significantly better reconstruction for 38 and comparable results for the remaining 5.
Although the runtime of both methods depends heavily on the number of iterations, random initializations, and points sampled, we observe a $7\times$ speedup when replacing UGW with JGW under identical hyperparameters. 
This improvement arises because JGW solves a single alignment subproblem instead of multiple one-by-one subproblems, and its approximation is slightly faster to compute than that of UGW~\cite{sejourne2021unbalanced}.

\section*{Conclusion}
In this paper, we formulate a novel variant of the Gromov-Wasserstein distance specifically designed to calculate a dissimilarity measure between two collections of mm-spaces, which we call the \emph{Joint Gromov-Wasserstein} (JGW) objective. We prove several theoretical properties of this new variant and analyze its behavior by showing useful results in the partial isomorphism and point sampling scenarios. Furthermore, we propose a method to adapt existing algorithms designed for computing the entropic regularized Gromov-Wasserstein distance to approximate the solution of our formulation in practice. Extensive experiments on partial matching, shape matching, and cryo-EM density map alignment tasks suggest that JGW is applicable to a wide range of problems, significantly outperforming classical partial and unbalanced variants with particularly strong performance when matching multiple distributions. In particular, we suggest JGW as an effective alignment method for structural biology and atomic model-building applications, where multiple chains must be matched.

Among the potential directions for improving our method, we first mention that the choice of approximation algorithm significantly affects the quality of the transportation map, which is crucial for applications involving continuous data types such as shape matching and cryo-EM density map alignment. Our formulation of JGW can adapt most existing approximations and relaxations developed for the original GW distance using a similar approach, including the methods introduced in \cite{li2023convergent, chapel2020partial} and one of the three linear lower bounds in \cite{memoli2011gromov}. In the context of cryo-EM, while our results suggest that the present approach is suitable for model-building applications, further validation on additional structures with different approximation methods is needed. Finally, in a more practical scenario, it would be important to test the method's ability to handle heterogeneous alignments by aligning a source map with separately produced parts.

\section*{Acknowledgments}
This research is supported by a NSERC Discovery Grant RGPIN-2020-05348.

\bibliographystyle{plos2015}
\bibliography{main}
\newpage

\section*{Appendix}
\appendix

\section{Proof of Theorem \ref{thm:embedding-unique}} \label{app:uniqueness}
\begin{lemma1}\label{lemma:projection}
    Given a distribution of mm-spaces $\mathbf{X} = (X_i, d_{X_i}, \mu_{X_i}, s_{X_i})_{i\in[k_X]}$ and an embedding $(X, d_X, \mu_X)$ with embedding functions $(\psi_{X_i})_{i\in[k_X]}$, there exist a bijective function $\pi_X: \bigcup_i X_i \rightarrow X$, such that $\pi_X(x_i) = \psi_{X_i}(x_i)$ for all $x_i \in X_i$. We call this map a projection.
\end{lemma1}
\begin{proof}
    Define $\pi_X$ by $\pi_X(x_i) = \psi_{X_i}(x_i)$ for all $x_i \in X_i$. By $(ii)$ property of Definition \ref{def:embedding}, we know that the images of $\psi_{X_i}$s are distinct, thus $\pi_X$ is injective. Also using $(iii)$ property of Definition \ref{def:embedding}, we conclude that $\pi_X$ is surjective as well.
\end{proof}

\begin{proof}[Proof of Theorem \ref{thm:embedding-unique}]
    
Assume $X_1$ and $X_2$ (respectively $Y_1$ and $Y_2$) are two distinct embeddings for $\mathbf{X}$ ($\mathbf{Y}$) with embedding functions $(\psi_{X_{1,i}})_{i \in [k_X]}$, $(\psi_{X_{2,i}})_{i \in [k_X]}$ ($(\psi_{Y_{1,i}})_{i \in [k_Y]}$, $(\psi_{Y_{2,i}})_{i \in [k_Y]}$). 
Using Lemma \ref{lemma:projection}, we define the bijections $\pi_{X_1}, \pi_{X_2}, \pi_{Y_1}, \pi_{Y_2}$. Now let $\pi^*_X: X_1 \rightarrow X_2$ and $\pi^*_Y: Y_1 \rightarrow Y_2$ be defined by
\begin{equation}
\pi^*_X = \pi_{X_2} \circ \pi_{X_1}^{-1}, \ \pi^*_Y = \pi_{Y_2} \circ \pi_{Y_1}^{-1}. \nonumber\end{equation}
Now for $x, x' \in X_1$ and $y,y' \in Y_1$ by using the properties of projections in Lemma \ref{lemma:projection}, we want to prove that
\begin{equation}\Gamma^*_p(x,y,x',y') = \Gamma^*_p(\pi^*_X(x),\pi^*_Y(y),\pi^*_X(x'),\pi^*_Y(y')).\nonumber\end{equation}
To do so, we distinguish the following cases:
\begin{itemize}
    \item CASE 1: \emph{there exist $i \in [k_X]$ and $j \in [k_Y]$ such that $(x,x',y,y')\in Im(\psi_{X_{1,i}})^2\times Im(\psi_{Y_{1,j}})^2$.} By \eqref{eq:def-gammastar}, we can write
    \begin{align}
    \Gamma^*_p(x,y,x',y') &= |d_{X_1}(x,x') - d_{Y_1}(y,y')|^p = |d_{X_i}(\pi_{X_1}^{-1}(x),\pi_{X_1}^{-1}(x')) - d_{Y_j}(\pi_{Y_1}^{-1}(y),\pi_{Y_1}^{-1}(y'))|^p \nonumber\\&= |d_{X_i}(\pi^*_X(x),\pi^*_X(x')) - d_{Y_j}(\pi^*_Y(y),\pi^*_Y(y'))|^p 
    \label{eq:app1-gammastar-part1}
    \end{align}
    Also it is straightforward to see that if $(x,x',y,y')\in Im(\psi_{X_{1,i}})^2\times Im(\psi_{Y_{1,j}})^2$, then $(\pi^*(x),\pi^*(x'),\pi^*(y),\pi^*(y'))\in Im(\psi_{X_{2,i}})^2\times Im(\psi_{Y_{2,j}})^2$. By combining this fact with \eqref{eq:app1-gammastar-part1} we conclude that
    \begin{align}
        \Gamma^*_p(x,y,x',y') &= |d_{X_i}(\pi^*_X(x),\pi^*_X(x')) - d_{Y_j}(\pi^*_Y(y),\pi^*_Y(y'))|^p = \Gamma^*_p(\pi^*_X(x),\pi^*_Y(y),\pi^*_X(x'),\pi^*_Y(y')). \nonumber
    \end{align}
    \item CASE 2: \emph{otherwise:} Without  loss of generality, we can assume $x \in Im(\psi_{X_{1,i}})$ and $x' \in Im(\psi_{X_{1,i'}})$ with $i \ne i'$. Then
    \begin{align}
        x \in Im(\psi_{X_{1,i}}) \Rightarrow \pi_{X_1}^{-1}(x) \in X_i \Rightarrow \pi^*(x) = \pi_{X_2} \circ \pi_{X_1}^{-1}(x) \in Im(\psi_{X_{2,i}}), \nonumber
    \end{align}
    \begin{align}
        x' \in Im(\psi_{X_{1,i'}}) \Rightarrow \pi_{X_1}^{-1}(x') \in X_i' \Rightarrow \pi^*(x') = \pi_{X_2} \circ \pi_{X_1}^{-1}(x') \in Im(\psi_{X_{2,i'}}).\nonumber
    \end{align}
    As a result there exists no $i \in [k_X]$, such that $(\pi^*(x), \pi^*(x')) \in Im(\psi_{X_{2,i}})$, therefore
    \begin{equation}
        \Gamma^*_p(x,y,x',y') = 0 =\Gamma^*_p(\pi^*_X(x),\pi^*_Y(y),\pi^*_X(x'),\pi^*_Y(y')).\nonumber
    \end{equation}
\end{itemize}

We finalize the proof by using \eqref{eq:gw-definition} and get
\begin{align*}
    \mathcal{GW}_{\Gamma_p^*,p}(X_1,Y_1) & = \underset{\mu \in \mathcal{M}(\mu_{X_1}, \mu_{Y_1})}{\text{inf}} \frac{1}{2} \biggl( \int_{X_1\times Y_1}\int_{X_1\times Y_1} \Gamma^*_p(x,y,x',y')\mu(dx\times dy)\mu(dx'\times dy')\biggr)^{1/p}
    \nonumber\\& = \underset{\mu \in \mathcal{M}(\mu_{X_1}, \mu_{Y_1})}{\text{inf}} \frac{1}{2} \biggl( \int_{X_1\times Y_1}\int_{X_1\times Y_1} \Gamma^*_p(\pi^*_X(x),\pi^*_Y(y),\pi^*_X(x'),\pi^*_Y(y')) \nonumber \\ & \hspace{220pt} \mu(dx\times dy)\mu(dx'\times dy')\biggr)^{1/p}.
    \nonumber
\end{align*}

To simplify this equation, we perform the change of variables $x^* = \pi^*_X(x), y^* = \pi^*_Y(y), x'^{*} = \pi^*_X(x') \text{ and } y'^* = \pi^*_Y(y')$ and we denote $\mu^*(.,.) = \mu(\pi^{*-1}_X(.), \pi^{*-1}_Y(.))$ the image of $\mu$ under the product mapping $(\pi^*_X, \pi^*_Y)$. 

Therefore, using \cite[Theorem 3.6.1]{bogachev2007measure}, we conclude that

\begin{align*}
    \mathcal{GW}_{\Gamma_p^*,p}(X_1,Y_1) & = \underset{\mu \in \mathcal{M}(\mu_{X_1}, \mu_{Y_1})}{\text{inf}} \frac{1}{2} \biggl( \int_{X_1\times Y_1}\int_{X_1\times Y_1} \Gamma^*_p(\pi^*_X(x),\pi^*_Y(y),\pi^*_X(x'),\pi^*_Y(y')) \nonumber \\ & \hspace{220pt} \mu(dx\times dy)\mu(dx'\times dy')\biggr)^{1/p}
    \nonumber\\& = \underset{\mu \in \mathcal{M}(\mu_{X_1}, \mu_{Y_1})}{\text{inf}} \frac{1}{2} \biggl( \int_{X_2\times Y_2}\int_{X_2\times Y_2} \Gamma^*_p(x^*,y^*,x'^*,y'^*)\mu^*(dx\times dy)\mu^*(dx'\times dy')\biggr)^{1/p}.
    \nonumber
\end{align*}

Using the properties of embedding functions, it's straightforward to check that
\begin{align*}
    \mu^* \in \mathcal{M}(\mu_{X_2}, \mu_{Y_2}) \iff 
    \mu \in \mathcal{M}(\mu_{X_1}, \mu_{Y_1}).
\end{align*}

Thus, we can change the domain of infimum and complete the proof as follows.

\begin{align*}
    \mathcal{GW}_{\Gamma_p^*,p}(X_1,Y_1) & =  \underset{\mu \in \mathcal{M}(\mu_{X_1}, \mu_{Y_1})}{\text{inf}} \frac{1}{2} \biggl( \int_{X_2\times Y_2}\int_{X_2\times Y_2} \Gamma^*_p(x^*,y^*,x'^*,y'^*)\mu^*(dx\times dy)\mu^*(dx'\times dy')\biggr)^{1/p}
    \nonumber\\&\quad\quad = \underset{\mu^* \in \mathcal{M}(\mu_{X_2}, \mu_{Y_2})}{\text{inf}} \frac{1}{2} \biggl( \int_{X_2\times Y_2}\int_{X_2\times Y_2} \Gamma^*_p(x^*,y^*,x'^*,y'^*)\mu^*(dx\times dy)\mu^*(dx'\times dy')\biggr)^{1/p}
    \nonumber\\&\quad \quad=\mathcal{GW}_{\Gamma_p^*,p}(X_2,Y_2).
\end{align*}


\end{proof}


    \section{Proof of Theorem \ref{thm:isomorphism}} \label{app:proofs}
\begin{lemma1}
    \label{lemma:lemmin}
    Given two distribution of mm-spaces $\mathbf{X}$ and $\mathbf{Y}$ with embbedings $X$ and $Y$ respectively, there exists a coupling $\mu^* \in \mathcal{M}(\mu_X, \mu_Y)$ such that
    \begin{equation}
        \nonumber
        \mathcal{JGW}_{p}(\mathbf{X},\mathbf{Y}) = \frac{1}{2} \left( \mathcal{D}(\mu^*)\right)^{1/p},
    \end{equation}
    where
    \begin{align}
    \label{eq:d-definition}
        &\mathcal{D}(\mu^*) = \int_{X\times Y}\int_{X\times Y} \quad\quad \Gamma^*_p(x,y,x',y')\mu^*(dx\times dy)\mu^*(dx'\times dy').
    \end{align}
\end{lemma1}

\begin{proof}
    For this we need to show the sequential compactness of $\mathcal{M}(\mu_X,\mu_Y)$ and the continuity of $\mathcal{D}$. The former is provided in \cite[p. 49]{villani2021topics} and the latter follows from \cite[Lemma 10.3]{memoli2011gromov}.
\end{proof}

\begin{proof}[Proof of Theorem \ref{thm:isomorphism}]
    Let $X$ and $Y$ be embeddings for $\mathbf{X}$ and $\mathbf{Y}$ with embedding functions $\psi^X_i$ and $\psi^Y_j$ respectively, and $\mathcal{D}$ be defined similarly as in \eqref{eq:d-definition}.
    
    For the ``if" part, we want to show that if $\mathbf{X}$ and $\mathbf{Y}$ are partially isomorphic, then there exists a coupling $\mu^* \in \mathcal{M}(\mu_X, \mu_Y)$ such that $\mathcal{D}(\mu^*) = 0$. By definition, there exist mm-spaces $Z_{i,j}$ with isometry functions $\psi^X_{i,j}: Z_{i,j} \rightarrow X_i, \psi^Y_{i,j}: Z_{i,j} \rightarrow Y_j$ satisfying the conditions specified in Definition \ref{def:partialisomorphism}.
    For simplicity, given $x \in X, y \in  Y$, we define 
    \begin{align}
    \nonumber
        &\mathcal{Z}(x,y) = \{\mu_{Z_{i,j}}(z) \  \vert \ \exists i \in [k_X], j\in [k_Y],  z\in Z_{i,j} \ \text{ s.t. } \psi^X_i \circ \psi^X_{i,j}(z) = x, \psi^Y_j \circ \psi^Y_{i,j}(z) = y\}.
    \end{align}
    Now using this operator define $\mu^*(x,y)$ as 
    \begin{equation}\nonumber
        \mu^*(x,y) = 
        \begin{cases}
            \sum_{p \in \mathcal{Z}(x,y)} p& \text{if } \mathcal{Z}(x,y) \ne \emptyset
            \\\
            0&\text{else}
        \end{cases}.
    \end{equation}
    Now consider $(x_1,y_1)$ and $(x_2,y_2)$ such that $\mu^*(x_i,y_i)\neq 0$ for $i=1,2$.
    Since $\mu^*(x_1,y_1) \ne 0$, $\mathcal{Z}(x_1,y_1) \ne \emptyset$ and there exists $z_1, i_1, j_1$ such that $\psi^X_{i_1} \circ \psi^X_{i_1,j_1}(z_1) = x_1, \psi^Y_{j_1} \circ \psi^Y_{i_1,j_1} (z_1) = y_1$. By the same argument, we similarly define $z_2, i_2, j_2$. If $i_1 \ne i_2$  or $j_1 \neq j_2$, since the images of $\psi^X_i$s are disjoint (due to the properties of embedding functions in Definition \ref{def:embedding}), then $\Gamma^*_p(x_1,y_1,x_2,y_2) =0$. If $i_1 = i_2$, and $j_1=j_2$, we have the isometries $\psi^X_{i_1}, \psi^X_{i_1,j_1}, \psi^Y_{j_1}, \psi^Y_{i_1,j_1}$, and we can write
    \begin{align}
        \nonumber
        d_{Z_{i_1,j_1}}(z_1,z_2) &= d_{X_{i_1}}(\psi^X_{i_1,j_1}(z_1),\psi^X_{i_1,j_1}(z_2)) \\
        \nonumber
        &=d_{X}(\psi^X_{i_1} \circ \psi^X_{i_1,j_1}(z_1),\psi^X_{i_1} \circ \psi^X_{i_1,j_1}(z_2))\\
        \nonumber
        &= d_X (x_1, x_2),\\
        \nonumber
        d_{Z_{i_1,j_1}}(z_1,z_2) &= d_{Y_{j_1}}(\psi^Y_{i_1,j_1}(z_1),\psi^Y_{i_1,j_1}(z_2)) \\ 
        \nonumber
        &=d_{Y}(\psi^Y_{j_1} \circ \psi^Y_{i_1,j_1}(z_1),\psi^Y_{j_1} \circ \psi^Y_{i_1,j_1}(z_2))\\
        \nonumber
        &= d_Y (y_1, y_2).
    \end{align}
    This implies $d_X(x_1,x_2) = d_Y(y_1, y_2)$, so we showed that $
        \Gamma^*_p(x_1,y_1,x_2,y_2) = 0
    $ and as a result $\mu^*$ satisfies  $\mathcal{D}(\mu^*) = 0$.
    
    
    For the other direction assume $\mathcal{JGW}_p(\mathbf{X}, \mathbf{Y}) = 0$. Using     Lemma \ref{lemma:lemmin}, there exists $\mu^*$ such that $\mathcal{D}(\mu^*) = 0$. 
    Now for each $i \in [k_X]$ and $j \in [k_Y]$ define $Z_{i,j}$ as the set of couplings 
    \begin{equation}
        Z_{i,j} = \{(x, y) \vert x \in X_i, y \in Y_j, \mu^*\left(\psi^X_i(x), \psi^Y_j(y)\right) \ne 0 \}, \nonumber
    \end{equation}
    endowed with a measure where we assign $\mu^*\left(\psi^X_i(x), \psi^Y_j(y)\right)$ to $(x,y)$ and normalize it. Now for $(x_1,y_1), (x_2,y_2) \in Z_{i,j}$, since $\mathcal{D}(\mu^*)=0$ and $\psi^X_i, \psi^Y_j$ are isometries, we have $d_{X_i}(x_1,x_2) = d_{Y_j}(y_1,y_2)$. 
    Thus $Z_{i,j}$ can be equipped with a metric function (either $d_{X_i}$ or $d_{Y_j}$) and forms a mm-space, such that 
    \begin{align*}
Z_{i,j} & \rightarrow X_i \\
\psi_{i,j}^X : (x,y)& \mapsto x
 \nonumber
    \end{align*}
    and 
     \begin{align*}
Z_{i,j} & \rightarrow Y_j \\
\psi_{i,j}^Y : (x,y)& \mapsto y
 \nonumber
    \end{align*} 
    are isometries. Using the assumption that $\mathcal{D}(\mu^*)= 0$ and the isomorphism properties of $\psi^X_i$ and  $\psi^Y_j$, we can further verify that $d_{X_i}$ and $d_{Y_J}$ satisfy the conditions for partial isomorphism, and thus
    $Z_{i,j}$ provides the desired partial isomorphism between $\mathbf{X}, \mathbf{Y}$.
    
\end{proof}

\section{Proof of Theorem \ref{thm:sampling}} \label{app:proof-samp}

\begin{proof}[Proof of Theorem \ref{thm:sampling}]
Let $X$ be an embedding of $\mathbf{X}$ with embedding functions $\psi_i: X_i \rightarrow X$, and $\{x_j\}_{j\in[n]}$ be  $n$ points in $\bigcup_{i \in [k_X]} X^n_i$. Let $X^n = \{\psi_i (x_{j}) \vert \forall j \in [n], x_j \in X_i\}$ endowed with the uniform empirical measure on it. As $\psi_i$s hold the properties of embedding functions, $X^n$ is an embedding for $\mathbf{X}^n$. Therefore using \eqref{eq:jgw-definition} we have
\begin{equation}\label{eq:step1}
    \mathcal{JGW}_p(\mathbf{X}, \mathbf{X}^n) = \mathcal{GW}_{\Gamma^*_p, p}(X, X^n).
\end{equation}

By definition of $\Gamma^*$, one can see that $\Gamma^*(x,y,x',y') \le \Gamma(x,y,x',y')$ for all $x,y,x',y'$. Hence,
\begin{equation}
\label{eq:step3}
\mathcal{GW}_{\Gamma^*_p, p}(X, X^n) \le \mathcal{GW}_{\Gamma_p, p}(X, X^n).    
\end{equation}
A similar statement for $\mathcal{GW}_{\Gamma_p,p}$ (\cite{memoli2011gromov}, Theorem 5.1e), shows that $\mathcal{GW}_{\Gamma_p, p}(X,X^n)$ almost surely converges to zero as $n\rightarrow \infty$. The combination of this Theorem with \eqref{eq:step1}, \eqref{eq:step3}, shows that $\mathcal{JGW}_p(\mathbf{X}, \mathbf{X}^n)$ almost surely converges to zero as $n \rightarrow \infty$.

\end{proof}

    \section{Proof of Proposition \ref{prp:regularization}}\label{app:approx}

\begin{proof}[Proof of Proposition \ref{prp:regularization}]
     
Using the definition of entropic Joint Gromov-Wasserstein objective ($\mathcal{JGW}^\epsilon(\mathbf{X}, \mathbf{Y})$), we can expand
\begin{align}
\sum_{i,j,k,l}\lvert d^X_{ij} - d^Y_{kl} \rvert^2 I^X_{ij}I^Y_{kl} \mu_{ik}\mu_{jl}&= \sum_{i,j,k,l} {d^X_{ij}}^2I^X_{ij}I^Y_{kl} \mu_{i,k}\mu_{j,l} - 2 \sum_{i,j,k,l}d^X_{i,j}d^Y_{k,l}I^X_{ij}I^Y_{kl}\mu_{ik}\mu_{jl}\nonumber\\&\quad \quad + \sum_{i,j,k,l}{d^Y_{k,l}}^2  I^X_{ij}I^Y_{kl} \mu_{i,k}\mu_{j,l} \nonumber\\
&=\sum_{i,k} \mu_{ik} \sum_{j,l} {d^X_{ij}}^2 \mu_{jl}I^Y_{lk} - 2 \sum_{i,k} \mu_{ik}\sum_{j,l}d^X_{ij}\mu_{jl}d^Y_{lk} \nonumber\\&\quad \quad + \sum_{i,k} \mu_{ik}\sum_{j,l}  I^X_{ij} \mu_{jl}{d^Y_{lk}}^2 \label{eq:simplification1} \\
&=\sum_{i,k} \mu_{ik} [d^{X\wedge2} \mu I^Y]_{ik} - 2 \sum_{i,k} \mu_{ik} [d^X \mu d^Y]_{ik}  \nonumber\\&\quad \quad + \sum_{i,k} \mu_{ik} [I^X \mu d^{Y\wedge2}]_{ik}\nonumber\\
&=\langle \mu, d^{X\wedge2} \mu I^Y \rangle - 2 \langle \mu,  d^X \mu d^Y\rangle + \langle \mu,  I^X \mu d^{Y\wedge2}\rangle \nonumber\\
&=\langle \mu, d^{X\wedge2} \mu I^Y  - 2  d^X \mu d^Y + I^X \mu d^{Y\wedge2}\rangle, \label{eq:jgw-entropic-simplification}
\end{align}
where $\langle .,.\rangle$ denotes the inner product of two given matrices and the superscript $\wedge2$ denotes the elementwise square of a matrix.
We used the fact that $d^X \otimes I^X = d^X$, and $d^Y \otimes I^Y = d^Y$ and $d^X, I^X, d^Y, I^Y$ are all symmetric in line \eqref{eq:simplification1}.
Combining \eqref{eq:jgw-entropic-simplification} with the definition of $\mathcal{JGW}^\epsilon(\mathbf{X}, \mathbf{Y})$, we then get
\begin{equation}
    \mathcal{JGW}^\epsilon(\mathbf{X}, \mathbf{Y}) = \underset{\mu \in \mathcal{M}(\mu_X, \mu_Y)}{\text{min}} \frac{1}{2} \left( \langle \mu, \Lambda(\mu) \rangle  + \epsilon H(\mu)\right)^{1/2},\nonumber
\end{equation}
where $\Lambda$ is defined as
\begin{equation}
    \Lambda(\mu) = d^{X\wedge2} \mu I^Y  - 2  d^X \mu d^Y + I^X \mu d^{Y\wedge2}.\nonumber
\end{equation}

\end{proof}   

\clearpage


\newpage

\clearpage

\section*{Figures}

\begin{figure}[h!]
    \centering
    \includegraphics[width=0.99\linewidth]{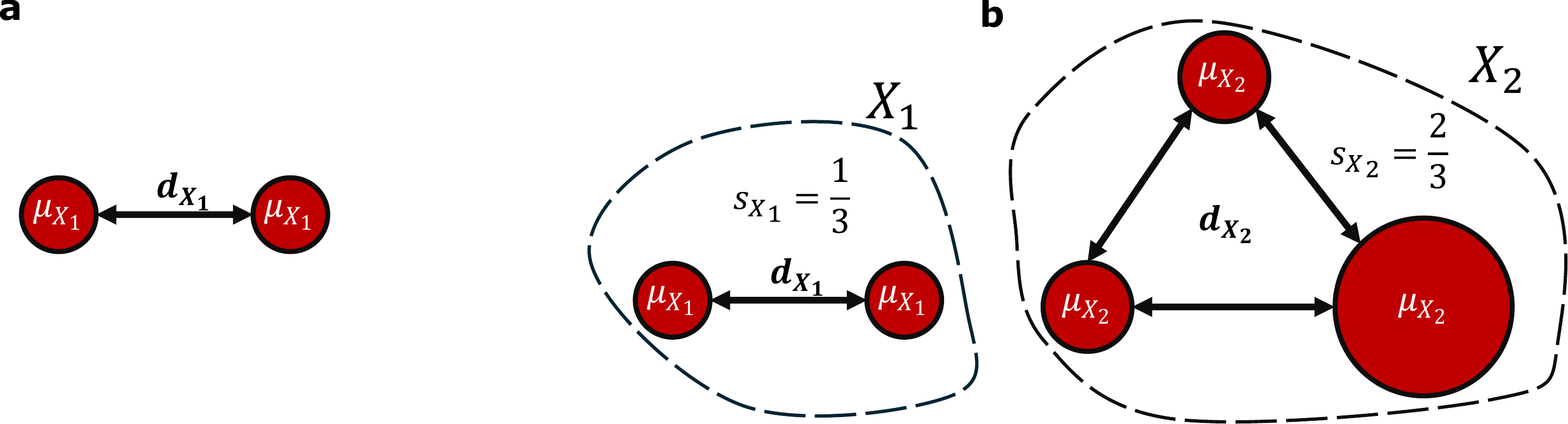}
    \caption{\textbf{a.} A simple example of a discrete mm-space with values of $d_x$ and $\mu_x$ provided. \textbf{b.} An example of a discrete distribution of mm-spaces containing two clusters with values of $d_x$, $\mu_x$, $s_x$ provided. Each point's size corresponds to the value of $\mu$ at that point.}
    \label{fig:example1}
\end{figure}\newpage

\begin{figure}[h!]
    \centering
    \includegraphics[width=0.99\linewidth]{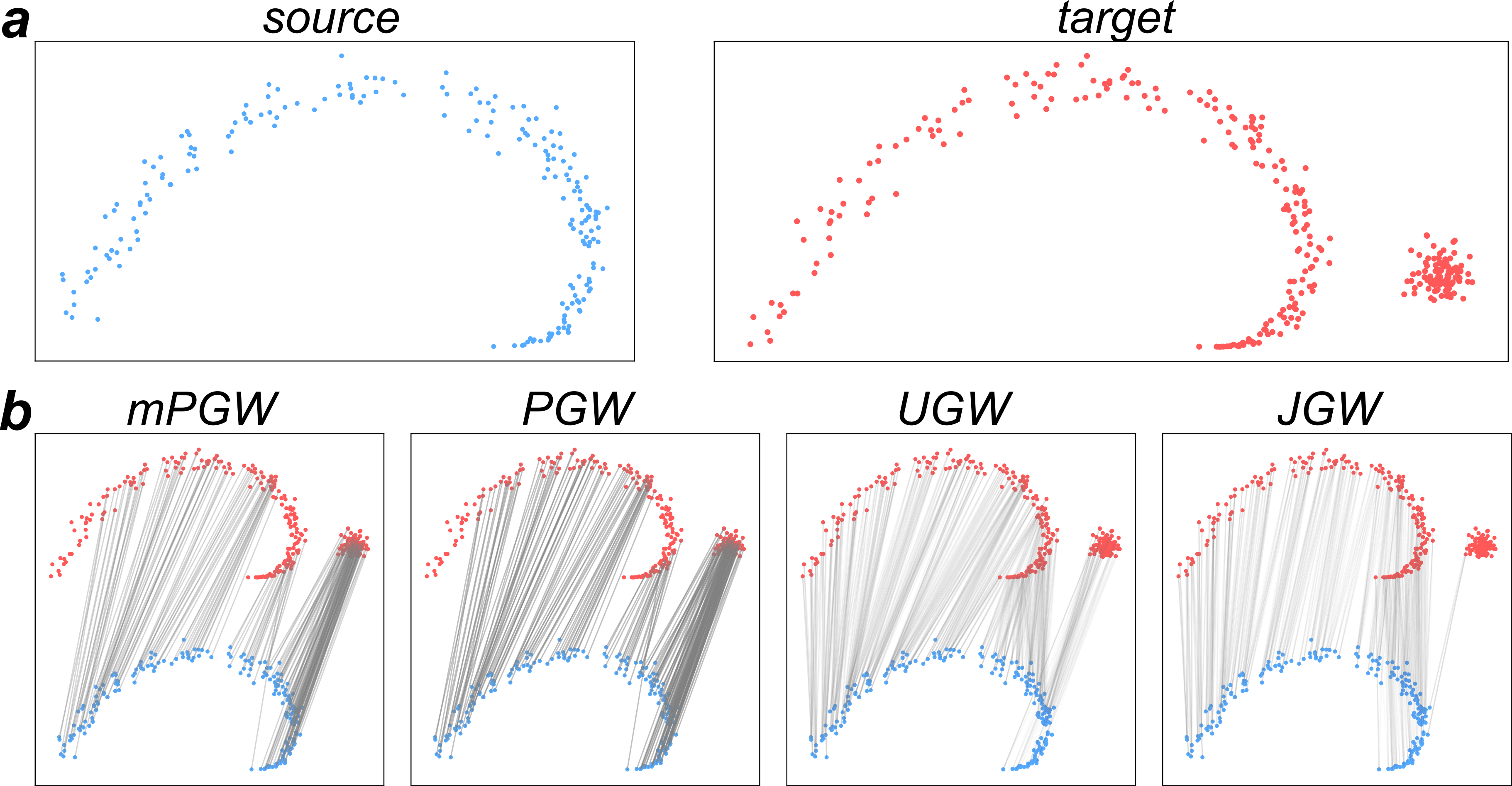}
    \caption{Performance comparison of GW variants for partial matching. We evaluate mPGW\cite{chapel2020partial}, PGW\cite{bai2024efficient}, UGW\cite{peyre2019computational}, and our proposed JGW approach. \textbf{a.} Source distribution (blue) comprising $200$ points sampled from an Archimedean spiral, and target distribution containing $200$ points from the same spiral plus $100$ noise points from a standard normal distribution (red). \textbf{b.} Couplings computed by each method, demonstrating JGW's superior performance in handling partial matches.}
    \label{fig:jgw-partial-coupling}
\end{figure}\newpage

\begin{figure}[h!]
    \centering
    \includegraphics[width=0.99\linewidth]{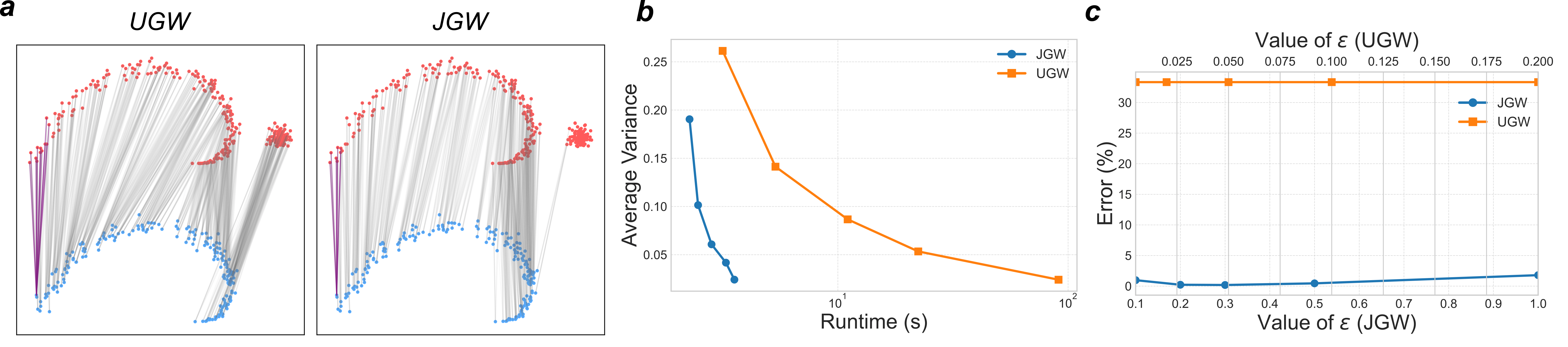}
    \caption{Comparison of the quality of the couplings generated by UGW and JGW on the same example as Figure \ref{fig:jgw-partial-coupling}. \textbf{a.} Couplings computed by each method, with visualization of how a single source point (the leftmost point in the source) is matched across the target distribution (purple edges). Both UGW and JGW distribute mass across multiple target points due to regularization, with JGW achieving lower variance. \textbf{b.} Runtime versus average per-point variance across regularization parameter values, confirming JGW's superior sparsity–runtime trade-off compared to UGW. \textbf{c.} Matching error across the regularization parameter ranges used previously, confirming the robustness of both methods within this range.
    }
    \label{fig:jgw-partial-quality}
\end{figure}\newpage

\begin{figure}[h!]
    \centering
    \includegraphics[width=0.99\linewidth]{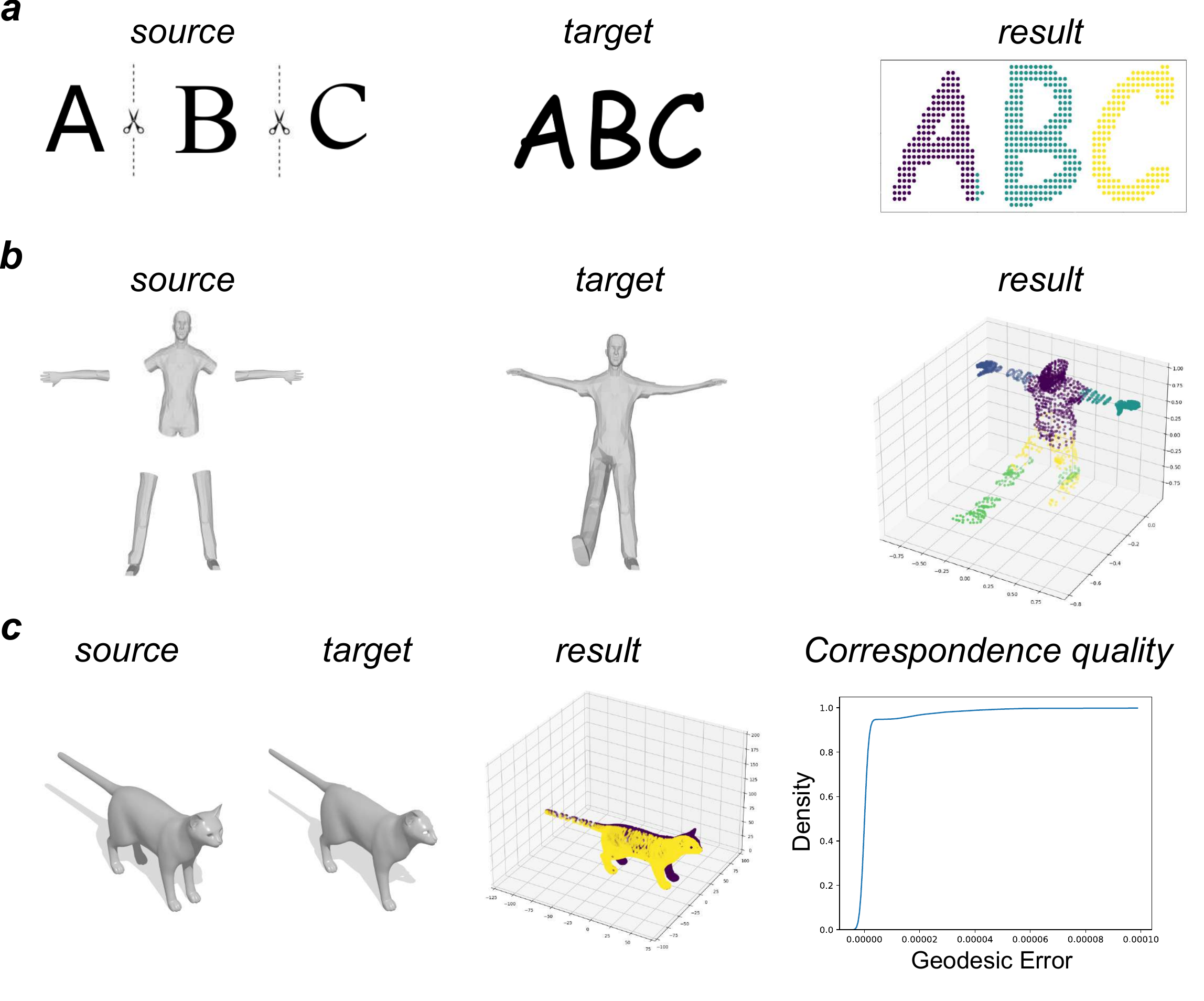}
    \caption{Performance of JGW in matching shapes involving 2D and 3D data. \textbf{a.} The source and target distributions created using different typesets and combinations of letters ``A", ``B", and ``C". Performance of JGW in matching the source space and target, each color shows the clusters of the coupled most points to a point of the target distribution.
    \textbf{b.} The source and target space created from 3D meshes of human body for CAPOD dataset \cite{papadakis2014canonically}. The results of the 3D experiments is demonstrated in the same way as before. This diagram shows the perfect performance of this method in matching the hands and the body, while mismatching some parts of the legs.
    \textbf{c.} The source and target space created from an example of SHREC'16 dataset \cite{cosmo2016shrec}. The results of the 3D experiments is demonstrated in the same way as before. This diagram shows the near-perfect performance of this method. To quantitatively evaluate the mapping quality, we employed a standard measure introduced in \cite{kim2011blended}: the geodesic distance between ground truth and computed corresponding points, normalized by the square root of the full shape's area, and illustrated a cumulative distribution function (CDF) of this measure across all mesh vertices.
    }
    \label{fig:jgw-shape}
\end{figure}\newpage

\begin{figure}[h!]
    \centering
    \includegraphics[width=0.99\linewidth]{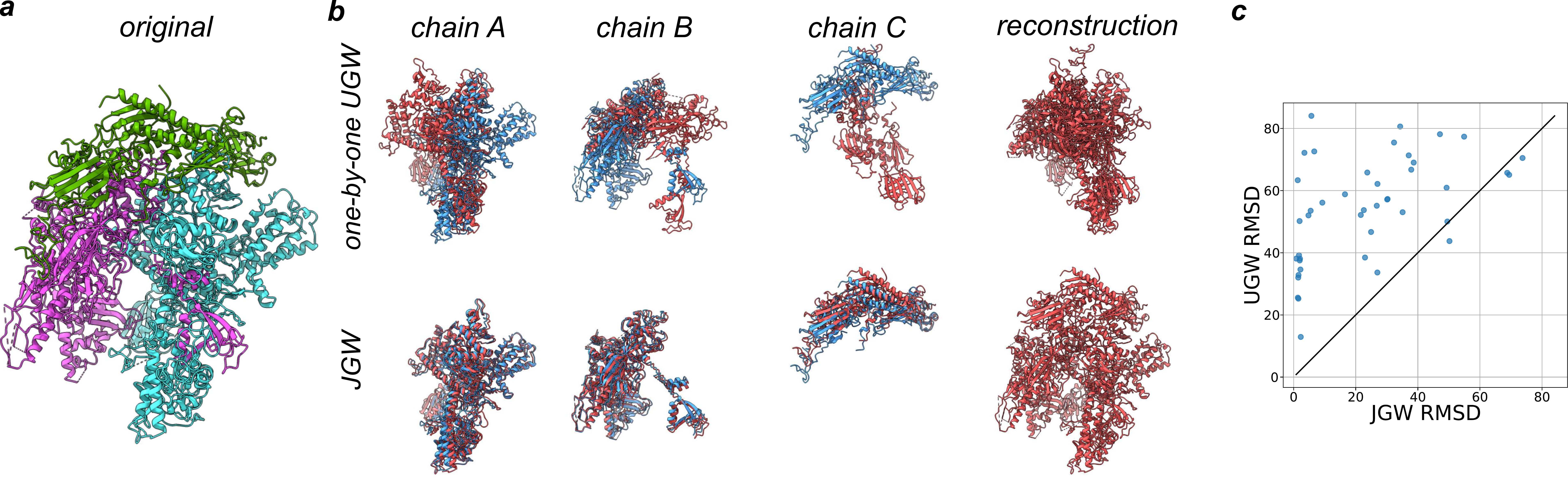}
    \caption{Performance of JGW on matching biomolecular complexes. \textbf{a.}Example of atomic structure (PDB:1I3Q \cite{cramer2001structural}) decomposed into 3 chains. \textbf{b.} The results of alignment of each chain and reconstruction using UGW (\cite{riahi2025alignment}) and one-by-one alignment of chains and JGW (\textit{ours}). In each diagram, the blue structure shows the ground truth while the red one represents the aligned one.
    \textbf{c.} Comparison of reconstruction accuracy (RMSD) between JGW and UGW across 43 protein models from \cite{he2022model}. Each point corresponds to one model, with its position indicating the RMSD achieved by each method.}
    \label{fig:bio}
\end{figure}\newpage
 \clearpage
 
\section*{Tables}

\begin{table}[h!]
    \centering
    \resizebox{\columnwidth}{!}{
    \begin{tabular}{ccccccc}
    \toprule
         metric & \multicolumn{2}{c}{Chain A} & \multicolumn{2}{c}{Chain B} & \multicolumn{2}{c}{Chain C} \\
          & JGW & UGW & JGW & UGW & JGW & UGW \\
         Rotational error & $\mathbf{5.1^\circ}$ & $30.4^\circ$ & $\mathbf{4.2^\circ}$   & $57.2^\circ$ & $\mathbf{8.9^\circ}$ & $161.9^\circ$\\
         RMSD & $\mathbf{3.018}$ & $25.974$ & $\mathbf{2.441}$  & $37.543$ & $\mathbf{8.275}$ & $63.858$\\
         \bottomrule
    \end{tabular}
    }
    \caption{Performance of JGW on matching biomolecular complexes compared to \cite{riahi2025alignment}. \textbf{a.} We used the atomic structure of PDB:1I3Q \cite{cramer2001structural} and simplified it into 3 chains. Then applied JGW and EMPOT to reconstruct it by aligning its chains into the whole map. We used 3 standard measurements to analyze the results of this experiment, and for each chain, highlighted the best result regarding each metric in bold. For all chains, significant improvement of all metrics is a consequence of a near-perfect alignment by JGW, as is illustrated in Figure \ref{fig:bio}.}
    \label{tab:bioresults}
\end{table}\newpage

\end{document}